\documentclass{article}
\usepackage{amsmath,graphicx}

\usepackage{microtype}
\usepackage{subfigure}
\usepackage{booktabs} 
\usepackage[matit]{echodefs}
\usepackage{lipsum}
\usepackage{amsfonts}
\usepackage{epstopdf}
\usepackage{mathtools}
\usepackage[autostyle]{csquotes}
\usepackage{amsmath,bm}
\usepackage{cleveref}
\usepackage{thmtools, thm-restate}
\usepackage{romannum}

\usepackage{algorithm}
\usepackage{algorithmic}

\newcommand{\calL}{\mathcal{L}}
\newcommand{\calX}{\mathcal{X}}
\newcommand{\E}{\mathbb{E}}

\title{Learning Schatten--von Neumann Operators}
\author{Puoya Tabaghi $^{\star}$ \qquad  Maarten de Hoop $^{\dagger}$ \qquad Ivan~Dokmani\'c $^{\star}$}
\date{}

\begin{document}
%
\maketitle
\begin{abstract}
We study the learnability of a class of compact operators known as Schatten--von Neumann operators. These operators between infinite-dimensional function spaces play a central role in a variety of applications in learning theory and inverse problems. We address the question of sample complexity of learning Schatten-von Neumann operators and provide an upper bound on the number of measurements required for the empirical risk minimizer to generalize with arbitrary precision and probability, as a function of class parameter $p$. Our results give generalization guarantees for regression of infinite-dimensional signals from infinite-dimensional data. Next, we adapt the representer theorem of Abernethy \emph{et al.} to show that empirical risk minimization over an a priori infinite-dimensional, non-compact set, can be converted to a convex finite dimensional optimization problem over a compact set. In summary, the class of $p$-Schatten--von Neumann operators is probably approximately correct (PAC)-learnable via a practical convex program for any $p < \infty$.
\end{abstract}

{\let\thefootnote\relax\footnote{{$^{\star}$ University of Illinois at Urbana-Champaign, USA $^{\dagger}$, CAAM, Rice University, Houston, USA,
$^{\star}${\small \tt \{tabaghi2, dokmanic\}@illinois.edu}, $^{\dagger}${\small \tt mdehoop@rice.edu}}}}

\section{Introduction} \label{Introduction}
Objects of interest in many problems in machine learning and inverse problems are best modeled as vectors in infinite-dimensional function spaces. This is the case in collaborative filtering in machine learning \cite{abernethy2009new}, non-linear inverse problems for the wave equation, and general regularized solutions to inverse problems \cite{engl1996regularization}. Relationships between these objects are often well-modeled by linear operators.

Among linear operators, compact operators are a natural target for learning because they are stable and appear commonly in applications. In infinite dimension, boundedness alone does not guarantee learnability (cf. Remark \ref{rmk:bounded}). Unbounded operators are poor targets for learning since they are not stable. A classical example from regularization of inverse problems is that if the direct operator is compact (for example, the Radon transform in computed tomography), then the inverse is unbounded, so we replace it by a compact approximation.

When addressing inverse and learning problems numerically, we need to discretize the involved operators. The fundamental properties of operator fitting, however, are governed by the continuous structure, and it is interesting to study properties of these continuous objects. 

Learning continuous operators from samples puts forward two important questions:
\begin{enumerate}
    \item Can a given class of operators be learned from samples? How many samples are needed to guarantee that we learn the ``best'' possible operator? Here we assume that the input samples $x$ and the output samples $y$ are drawn from some joint probability measure $P$ which admits a ``best'' $T$ that maps $x$ to $y$, and we ask whether $T$ can be approximated from samples from $P$.
    \item What practical algorithms exist to learn operators belonging to certain classes of compact operators, given that those classes are infinite-dimensional, and in fact non-compact? 
\end{enumerate}

In this paper we address these questions for compact operators known as $p$-Schatten--von Neumann operators, whose singular value sequences have finite $\ell^p$ norms. These operators find applications in a number of inverse problems, in particular those related to scattering. For the first question, we show that the class is learnable in the probably-approximately-correct sense for all $p < \infty$, and we prove the dependence of sample complexity on $p$. We work within the Vapnik-Chervonenkis framework of statistical learning theory, and bound the sample complexity via computing the Rademacher complexity of the class as a function of $p$.

For the second question, we adapt the results of by \cite{abernethy2009new} who showed that infinite-dimensional learning problems similar to ours can be transformed into finite-dimensional optimization problems. In our proofs, we make explicit the fact about the non-compactness of the involved hypothesis classes.

\subsection{Related work} 
\label{sub:related_work}

The closest work to ours is that of Maurer on sample complexity for multitask learning, \cite{maurer2006rademacher, maurer2006bounds}. He computes sample complexity of finite-rank operators $T : \mathcal{H} \to \R^m$, where $\mathcal{H}$ is a Hilbert space. In general, there is quite a bit of work on complexity of finite-dimensional classes. Kakade \emph{et al.} \cite{kakade2009complexity} study the generalization properties of scalar-valued linear regression on Hilbert spaces. A survey of statistical bounds for estimation and classification in terms of Rademacher complexity is available in \cite{mendelson2003few, boucheron2005theory}. To the best of our knowledge, this is the first paper to look at the sample complexity of learning infinite-dimensional operators, where the hypothesis class is non-compact.

On the algorithmic side Abernethy \emph{et al.} \cite{abernethy2009new} propose learning algorithms for a problem related to ours. They show how in the context of collaborative filtering, a number of existing algorithms can be abstractly modeled as learning compact operators, and derive a representer theorem which casts the problem as optimization over matrices for general losses and regularizers.

Our work falls under the purview of ``machine learning for inverse problems'', an expanding field of learning inverse operators and regularizers from data, especially with the empirical successes of deep neural networks. Machine learning has been successfully applied to problems in computed tomography \cite{jin2017deep, kothari2018random}, inverse scattering \cite{li2016underwater}, and compressive sensing \cite{bora:2017compressed}, to name a few. A number of parallels between statistical learning and inverse problems are pointed out in \cite{vito2005learning}.

\section{Motivation}

\subsection{Regularized Inverses in Imaging}

Consider a linear operator equation
\[
    y = Ax,
\]
where $A : \calX \to \calY$ is a Schatten-von Neumann operator. This is the case in a variety of ``mildly'' ill-posed inverse problems such as computed tomography, where $A$ is the Radon transform. It is well known that even when $A$ is injective, the compactness of $A$ makes the problem ill-posed in the sense that $A^{-1}$ is not a continuous linear operator from $\calY \to \calX$. This becomes problematic whenever instead of $y$ we get to measure some perturbed data $y^\delta$, as we always do. It also makes $A^{-1}$ a poor target for learning since it does not make much sense to learn unstable operators.

A classical regularization technique is then to solve
\[
    \wh{x} = \argmin_x \norm{y^\delta - Ax}^2 + \lambda \norm{x}^2,
\]
which can be interpreted as a maximum a posteriori solution under Gaussian noise and signal priors. The resulting solution operator $y^\delta \mapsto \wh{x}$ can formally be written as
\[
    R_\lambda = (A^* A + \lambda \cdot \text{Id})^{-1} A^*,
\]
which can be shown to be a Schatten-von Neumann operator. ($A^*$ denotes the adjoint of $A$.) If the forward operator $A$ is completely or partially unknown, or the noise and prior distributions are different from isotropic Gaussian, it is of interest to learn the best regularized linear operator from samples. Thus one of the central questions becomes that of sample complexity and generalization error.
 
\subsection{Collaborative Filtering}

In collaborative filtering the goal is to predict rankings of objects belonging to some class $\calY$ by users belonging to $\calX$. \cite{abernethy2009new} show that both the users and the objects are conveniently modeled as belonging to infinite-dimensional reproducing kernel Hilbert spaces $\calX$ and $\calY$. The rankings can then be modeled by the following functional on $\mathcal{X} \times \mathcal{Y}$,
\[
    [\text{ranking of}~y~\text{by}~x] = \inprod{x, Fy}.
\]
Given a training sample consisting of users and rankings embedded in their respective spaces, the ``best'' $F$ is estimated by minimizing regularized empirical risk. The regularizers used by \cite{abernethy2009new} are of the form $\Omega(F) = \sum_{k \geq 1} \eta_k(s_k(F))$, where $\set{s_k(F)}_{k \geq 1}$ are the singular values of $F$, and $\eta_k$ are non-decreasing penalty functions.

\subsection{Schatten--von Neumann Operators in Wave Problems}

Schatten--von Neumann operators play an important role in non-linear inverse problems associated with the wave equation. In particular, in inverse scattering approaches via boundary control \cite{bingham2008iterative} and scattering control \cite{caday2018reconstruction, caday2019scattering}, the reconstruction algorithms are given as recursive procedures with data operators that belong to the Schatten--von Neumann class. Under conditions that the inverse problem yields a unique solution, the question whether the inverse map is learnable may be analyzed by studying whether a Schatten-Von Neumann operator is learnable. While the overall inverse maps in these cases are nonlinear, we see the results presented here as a gateway to studying the learning-theoretic aspects of these general nonlinear inverse problems.

\section{Problem Statement and Main Result}\label{sec:main_result}

The main goal of this paper is to study the learnability of Schatten--Von Neumann class of compact operators. We use a model-free or agnostic approach \cite{HAUSSLER199278}, in the sense that we do not require our training samples to satisfy $y = Tx$ for any putative $T$, and the optimal risk can be nonzero. Instead, we are looking for an operator that provides the best fit to a given training set generated from an arbitrary distribution.

As usual, the training set consists of $N$ i.i.d. samples $z_1, \cdots, z_N \in \mathcal{Z}$, where $z_n = (x_n, y_n) \sim P$. The samples are generated from an unknown probability measure $P \in \mathcal{P}(\mathcal{Z})$, where $\calP(\mathcal{Z})$ is the set of all probability measures defined on $\mathcal{Z}$. We take $\mathcal{Z}$ to be the Cartesian product of input and output Hilbert spaces, namely $\mathcal{Z} = \mathcal{H}_x \times \mathcal{H}_y$. The hypothesis space $\mathcal{T}_p$ is the set of all $p$-Schatten--von Neumann operators:

\begin{definition}[Schatten--von Neumann] \label{def:schatten_operators}
If the sequence $(s_n)_{n \geq 1}$ of singular values of a compact linear operator $T: \mathcal{H}_x \rightarrow \mathcal{H}_y$ is $\ell^p$-summable, for $0 < p < \infty$,  then $T$ is said to belong to the $p$-Schatten-von Neumann class, denoted $S_p(\mathcal{H}_x, \mathcal{H}_y)$.
\end{definition}	

For a given sample $(x, y) \in \mathcal{H}_x \times \mathcal{H}_y$ and hypothesis $T \in \mathcal{T}_p$, we measure the goodness of fit, or the loss, as $\mathcal{L}(y, Tx) = \norm{y - Tx}^2$. The risk of a hypothesis $T$ is defined as its average loss with respect to the data-generating measure, $\mathbb{E}_P \mathcal{L}(y, Tx)$.

Ideally, we would then like to find an operator $T \in \mathcal{T}_p$ that has the minimum risk among the hypothesis class. While we do not have access to the true data-generating distribution, we get to observe it through a finite number of random samples, $(x_n, y_n) \sim P$, $n \in \set{1, \ldots, N}$.  Given a training set $\set{(x_n, y_n)}_{n=1}^{N}$, a central question is whether we can guarantee that the hypothesis learned from training samples will do well on other samples as well. Put differently, we want to show that the risk of the hypothesis learned from a finite  number of training samples is close to that of the hypothesis estimated with the knowledge of the true distribution, with high probability. 

We show that the class $\mathcal{T}_p$ is indeed learnable in the probably-approximately-correct (PAC) sense by showing that the risk achieved by the empirical risk minimization (ERM) converges to the true risk as the sample size grows. Concretely, assuming that $\norm{x_n}$ and $\norm{y_n}$ are bounded almost surely, we show that 

\begin{equation} \label{eq:PAC_learnability}
    \E_{P} \mathcal{L}(y,   \widehat{T}_N x) \leq  \inf_{T \in \mathcal{T}_p }\E_{P} \mathcal{L}(y, Tx) + O(N^{-\min \{ \frac{1}{p}, \frac{1}{2} \}})
\end{equation}
with high probability over training samples, where
\begin{equation*}
    \widehat{T}_N = \argmin_{T \in \mathcal{T}_p} \frac{1}{N}\sum_{n=1}^{N}{\mathcal{L}(y_n,T x_n)},
\end{equation*}
is a random minimizer of the empirical risk. (Existence of the empirical risk minimizer follows from \Cref{sec:learning_algorithm}.) We thus show that the class of $p$-Schatten--von Neumann operators is PAC-learnable, since ERM produces a hypothesis $\widehat{T}_N$ whose risk converges to the best possible. The rate of convergence determines the number of samples required to guarantee a given target risk with high probability.

The result makes intuitive sense. To see this, note that the inclusion
\[
    \mathcal{T}_{p_1} \subset \mathcal{T}_{p_2},
\]
for $p_1 < p_2$, implies that the smaller the $p$, the easier it is to learn $T \in \mathcal{T}_p$, as predicted by the theorem.

The minimization for $\wh{\mathcal{T}}_N$ involves optimization over an infinite-dimensional, non-compact set. In \Cref{sec:learning_algorithm} we show that in fact, the ERM can be carried out by a well-defined finite dimensional optimization.

\section{Learnability Theorem} \label{sec:learnability_theorem}

In order to state and prove the learnability thorem, we first recall some facts about linear operators in Hilbert spaces \cite{reed1980functional}. Then, we connect PAC-learnability of $p$-Schatten class of operators with its Rademacher complexity in \Cref{thm:learnability_theorem}. Finally, we establish the learnability claim by showing the said Rademacher complexity vanishes with the number of measured samples, \Cref{thm:vanishing_rademacher_complexity}.

\subsection{Preliminaries} \label{sec:schatten-von_neumann_operators}

Let $\mathcal{H}_x, \mathcal{H}_y$ be complex Hilbert spaces, $T : \mathcal{H}_x \rightarrow \mathcal{H}_y$ a compact operator, and let $T^*$ denote its adjoint operator. Let $\set{\psi_k}_k$ be the eigenvectors of a compact, self-adjoint, and non-negative operator $T^*T: \mathcal{H}_x \rightarrow \mathcal{H}_x$, and $\set{ \lambda_k(T^*T)}_{k}$ be the corresponding eigenvalues. We define the absolute value of $T$ as $|T| = (T^*T)^{1/2}$. The eigenvalues of $|T|$, $(\sqrt{\lambda_k(T^*T)})_{k\geq 1}$, are called the singular values of $T$ and denoted by $(s_k(T))_{k \geq 1}$. We always assume, without loss of generality, that the sequence of singular values is non-increasing.

Recall the definition of Schatten--von Neumann operators, \Cref{def:schatten_operators}. For $1 \leq p < \infty$, the class $S_p(\mathcal{H}_x, \mathcal{H}_y)$ becomes a Banach space equipped with the norm $\|T \|_{S_p}$,
\begin{equation} \label{eq:Sp_norm}
\|	T \|_{S	_p} = \bigg( \sum_{k=1}^{\infty}{s_k(T)^p}\bigg)^{1/p}.
\end{equation}
Specifically, $S_2(\mathcal{H}_x,\mathcal{H}_y)$ is the set of Hilbert-Schmidt operators while $S_1(\mathcal{H},\mathcal{H})$ is the algebra of trace class operators.  Let $T \in S_1(\mathcal{H}, \mathcal{H})$, or simply $T \in S_1(\mathcal{H})$, and $(\psi_k)_{k \geq 1}$ be any orthonormal basis for $\mathcal{H}$ (which exists because $\mathcal{H}$ is a Hilbert space). The series
\begin{equation} \label{eq:trace_norm}
    \mathrm{Tr}(T)  = \sum_{k\geq 1} { \langle T\psi_k,\psi_k \rangle}
\end{equation}
is well-defined and called the trace of $T \in  S_1(\mathcal{H})$. Finally, $S_{\infty}(\mathcal{H}_x, \mathcal{H}_y)$ is the class of bounded operators endowed with the operator norm $\|	T \|_{S_\infty} := \|T \|_{\text{op}}$.

An important fact about the set of bounded Schatten operators $\{T \in S_p(\mathcal{H}_x, \mathcal{H}_y) \ : \ \norm{T}_{S_p} \leq B\}$ is that it is not compact, which requires a bit of care when talking about the various minimizers. For $0 < p < 1$, $S_p(\mathcal{H}_x, \mathcal{H}_y)$ is a complete space with quasi-norm of $\| T \|_{S_p}$. 

\subsection{Rademacher Complexity and Learnability of the Hypothesis Class} \label{sec:radmad}

In this section, we show that ERM algorithm PAC-learns the class of $p$-Schatten--von Neumann operators. The proofs of all formal results are given in Section \ref{sec:proofs}. 

Informally, ERM works if we have enough measurements to dismiss all bad near-minimizers of the empirical risk. In other words, generalization ability of the ERM minimizer is negatively affected by the complexity of the hypothesis class. A common complexity measure for classes of real-valued functions is the Rademacher complexity  \cite{shalev2014understanding, bartlett2002rademacher}.

\begin{definition}[Rademacher Complexity]\label{def:rademacher_complexity}
Let $P$ be a probability distribution on a set $\mathcal{Z}$ and suppose that $Z_1, \cdots, Z_N \in \mathcal{Z}$ are independent samples selected according to $P$. Let $\mathcal{F}$ be a class of functions mapping from $\mathcal{Z}$ to $\R$. Given samples $Z^N := Z_1,\cdots, Z_N$, the conditional Rademacher complexity of $\mathcal{F}$ is defined as
\begin{equation*}
R_N(\mathcal{F}|Z^N) = \E\big[ \sup_{f \in \mathcal{F}} | \frac{1}{N} \sum_{n=1}^{N}{\sigma_i f(Z_n)}|  \big| Z^N \big] 
\end{equation*}
where $\sigma_1, \cdots, \sigma_N$ are i.i.d Rademacher random variables. The Rademacher complexity of $\mathcal{F}$ is $R_N(\mathcal{F}) = \E R_N(\mathcal{F}|Z^N)$.
\end{definition}

The following result connects PAC-learnability of our operators with the Rademacher complexity the hypothesis class. It shows that a class is PAC-learnable when Rademacher complexity converges to zero as the size of the training dataset grows.

\begin{restatable}[Learnability]{theorem}{main}
\label{thm:learnability_theorem}
Let $(x_n,y_n)_{n=1}^{N}$ be independently selected according to the probability measure $P \in \mathcal{P}(\mathcal{Z})$. Moreover, assume $\| x_n \|$ and $\|y_n \|$ are bounded random variables, $\| x_n \| \leq C_x$ and $\| y_n \| \leq C_y$ almost surely. Then, for any $0 < \delta < 1$, with probability at least $1-\delta$ over samples of length $N$, we have
\begin{align*}
\mathbb{E} \mathcal{L}(y, &\widehat{T}_N x) \leq  \inf_{T \in \mathcal{T}_p }\mathbb{E} \mathcal{L}(y, Tx) + 4 R_N( \mathcal{L} \circ \mathcal{T}_p ) + 2(C_y + BC_x)^2 \sqrt{\frac{1}{N}  \log \frac{1}{\delta} } 
\end{align*}
where $\mathcal{T}_p = \{ T \in S_p(H_x,H_y): \| T \|_{S_p} \leq B \}$, $\widehat{T}_N = \argmin_{T \in \mathcal{T}_p} \frac{1}{N}\sum_{n=1}^{N}{\mathcal{L}(y_n,T x_n)}$, and $\mathcal{L} \circ \mathcal{T}_p = \{f:(x,y) \rightarrow \mathcal{L}(y,Tx): T \in \mathcal{T}_p \} $, and $1 \leq p < \infty$. 
\end{restatable}

The proof is standard; we adapt it for the case of Schatten--von Neumann operators.

Our main result is the following bound on the Rademacher complexity of the loss class induced by $\mathcal{T}_p$.

\begin{restatable}[Vanishing Rademacher Complexity]{theorem}{RC}
\label{thm:vanishing_rademacher_complexity}
Let $(x_n, y_n)_{n=1}^{N}$ be independently selected according to the probability measure $P \in \mathcal{P}(\mathcal{Z})$ and $\mathcal{T}_p = \{ T \in S_p(H_x,H_y): \| T \|_{S_p} \leq B \}$.  Moreover, assume $\| x_n \| \leq C_x$ and $\|y_n \| \leq C_y$ almost surely. Then, we have
\begin{equation*}
R_N(\mathcal{L} \circ  \mathcal{T}_p) \leq \frac{C_y}{\sqrt{N}}+ N^{-\min\{ \frac{1}{2}, \frac{1}{p}\}}(B^2 C_x  + 2BC_x C_y)
\end{equation*}
where $\mathcal{L} \circ \mathcal{T}_p = \{f:(x,y) \rightarrow \mathcal{L}(y,Tx): T \in \mathcal{T}_p \} $, $\mathcal{L}(y,T x) = \|y -Tx\|^{2}$, and $1 \leq p < \infty$.
\end{restatable}

\begin{remark}
    \label{rmk:bounded}
    As already mentioned, the bound becomes worse as $p$ grows large, and it breaks down for $p \to \infty$. This makes intuitive sense: $p = \infty$ only tells us that the singular values are bounded. It includes situations where all singular values are equal to $B$ and thus approximating any finite number of singular vectors leaves room for arbitrarily large errors.
\end{remark}

\begin{remark}
    The bound saturates for $p = 2$ and does not improve for $p < 2$. While this is likely an artifact of our proof technique, in terms of generalization error it is order optimal, due to the last $O(N^{-1/2})$ term in Theorem \ref{thm:learnability_theorem}.
\end{remark}

The bounds from Theorem \ref{thm:learnability_theorem} are illustrated in\Cref{fig:excess_risk}.

Our proof of Theorem \ref{thm:vanishing_rademacher_complexity} uses properties of certain random finite-rank operators constructed from training data. 

For notational convenience, let us define the operator $x_a x_b^*: \mathcal{H}_b \rightarrow \mathcal{H}_a$ such that for every $x \in \mathcal{H}_b$, we have  $x_a x_b^* x = \langle x_b ,  x \rangle x_a $. To prove Theorem \ref{thm:vanishing_rademacher_complexity}, we have to show that the operators $\sum_{n=1}^{N}{\sigma_n x_n x_n^*}$ and $\sum_{n=1}^{N}{\sigma_n y_n x_n^*}$, with $\sigma_n$ being i.i.d Rademacher random variables, are well-behaved in the sense that they do not grow too fast with the number of training samples $N$. This is the subject of the following lemma.

\begin{lemma}
\label{lem:random-op}
Let $(x_n,y_n)_{n=1}^{N}$ be independently selected according to the probability measure $P \in \mathcal{P}(\mathcal{Z})$ and $q \geq 1$. Define random operators $T_{xx}  := \sum_{n=1}^{N}{\sigma_n x_n x_n^*}$ and $T_{yx}  := \sum_{n=1}^{N}{\sigma_n y_n x_n^*}$. We have
\begin{align*}
&\E \| T_{xx}\|_{S_q} \leq  N^{\max\{ \frac{1}{2}, \frac{1}{q}\}}\sqrt{ \E  \| x\|^4  }, \\
&\E \|T_{yx} \|_{S_q} \leq  N^{\max\{ \frac{1}{2}, \frac{1}{q}\}}\sqrt{ \E  \| x\|^2 \| y\|^2  }.
\end{align*}
\end{lemma}

These growth rates then imply the stated bounds on the Rademacher complexity.

\begin{figure}
  \includegraphics[width=.7\linewidth]{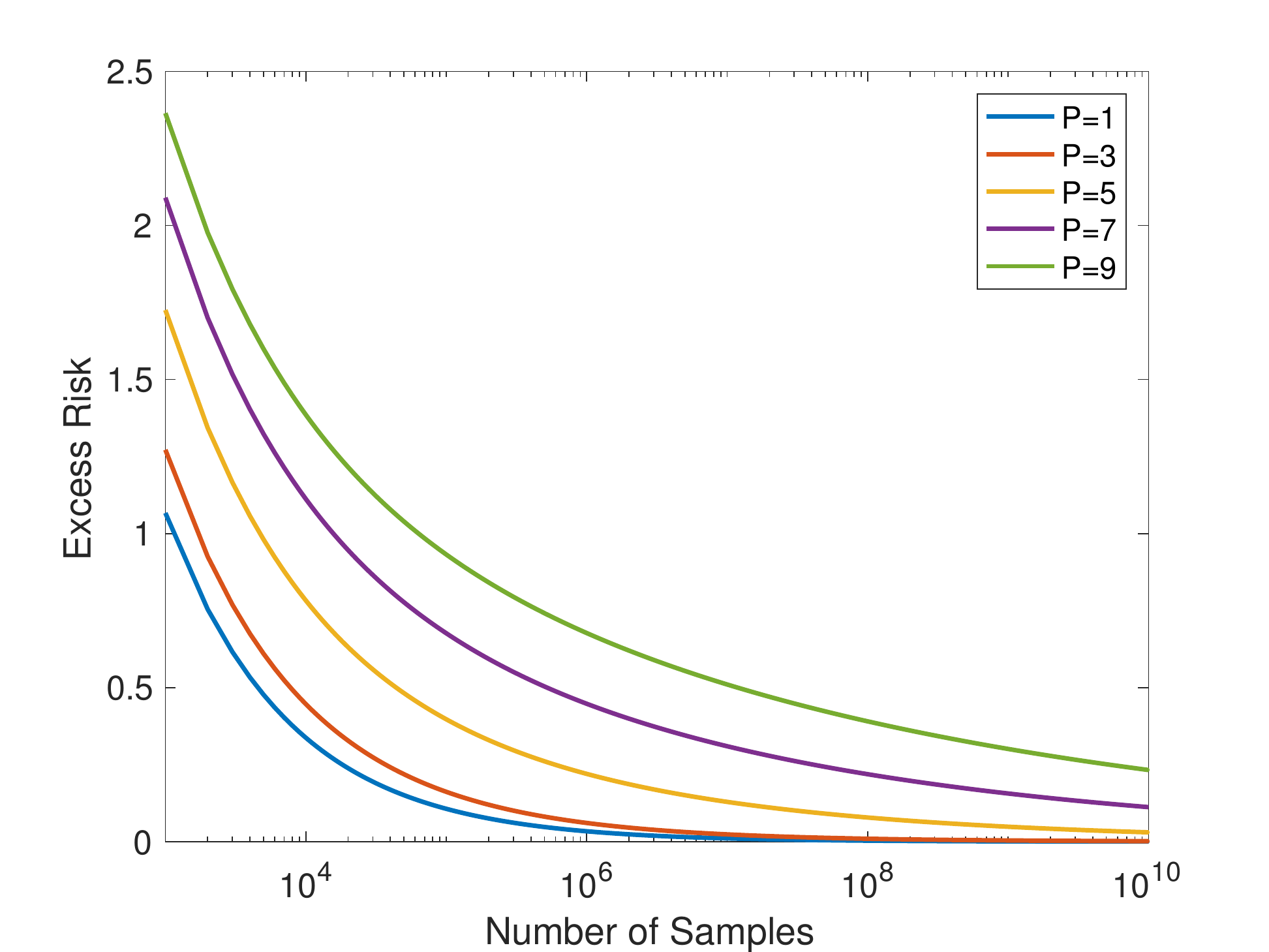}
  \centering
  \caption{Excess risk vs. the number of samples required to estimate $T \in \mathcal{T}_p$, with $\delta = 10^{-3}$.}
  \label{fig:excess_risk}
\end{figure}

\section{The Learning Algorithm} \label{sec:learning_algorithm}

In the developments so far we have not discussed the practicalities of learning infinite-dimensional operators. The ERM for $\wh{T}_N$ assumes we can optimize over an infinite-dimensional, non-compact class. We address this problem by adapting the results of \cite{abernethy2009new}, who show for a different loss function that ERM for compact operators can be transformed into an optimization over rank-$N$ operators.

Denote by $\mathcal{X}$ the linear span of $\set{x_1, \ldots, x_N}$, and by $\mathcal{Y}$ the linear span of $\set{y_1, \ldots, y_N}$. Let further $\Pi_{\mathcal{X}}$ be a projection onto $\mathcal{X}$ and analogously for $\mathcal{Y}$. Then, we have the following simple result:

\begin{lemma}
\label{lem:projected_loss}
Let $\Pi_\mathcal{X}: \mathcal{H}_x \rightarrow \mathcal{X}$ and $\Pi_\mathcal{Y}: \mathcal{H}_y \rightarrow \mathcal{Y}$ be projection operators. Then, we have $\calL(y, \Pi_\mathcal{Y} T \Pi_\mathcal{X} x) \leq \calL(y, Tx)$ for any operator $T:\mathcal{H}_x \rightarrow \mathcal{H}_y$ and all $(x, y) \in \mathcal{X} \times \mathcal{Y}$.
\end{lemma}

\begin{proof}
Let $(x, y) \in \mathcal{X} \times \mathcal{Y}$. We have
\begin{align*}
    \calL(y, \Pi_\mathcal{Y} T \Pi_\mathcal{X} x) 
    &= \norm{y - \Pi_\mathcal{Y} T \Pi_\mathcal{X} x}^2 \\
    &= \norm{\Pi_\mathcal{Y} (y - T x) }^2 \\
    &\leq \calL(y,Tx).
\end{align*}
\end{proof}
\Cref{lem:projected_loss} implies that for any $T \in \mathcal{T}_p$
\[
    J(\Pi_\mathcal{Y} T \Pi_\mathcal{X}) \leq J(T),
\]
where $J(T) = \frac{1}{N}\sum_{n=1}^{N}{\mathcal{L}(y_n,T x_n)}$ is the empirical risk of $T$. 

We use this simple result to show that the $\argmin$ in ERM is well-defined and can be achieved by solving a finite dimensional convex program over a compact set.

By definition of $\inf$, for any $\epsilon > 0$ there exists $\widehat{T}_{\epsilon} \in \mathcal{T}_p$ such that
\begin{equation*}
\inf_{T \in \mathcal{T}_p} J(T) \leq J(\widehat{T}_{\epsilon}) \leq \inf_{T \in \mathcal{T}_p} J(T) + \epsilon.
\end{equation*}
Let us choose $\widehat{T}_{\epsilon_n} \in \mathcal{T}_p$ for a sequence $(\epsilon_n)$ such that $\epsilon_n \to 0$. By construction, we have
\begin{equation*}
\lim_{n \rightarrow +\infty} J(\widehat{T}_{\epsilon_n}) = \inf_{T \in \mathcal{T}_p}  J(T).
\end{equation*}
From \Cref{lem:projected_loss},
 \begin{equation}\label{eq:T_seq}
\lim_{n \rightarrow +\infty} J(\Pi_\mathcal{Y} \widehat{T}_{\epsilon_n} \Pi_\mathcal{X}) \leq \inf_{T \in \mathcal{T}_p}  J(T).
\end{equation}

On the other hand, for a compact operator $T:\mathcal{H}_x \rightarrow \mathcal{H}_y$, we have \cite{abernethy2009new}
\begin{equation*}
\sigma_n(\Pi_\mathcal{Y} T \Pi_\mathcal{X}) \leq \sigma_n(T), ~ \forall n \in \mathbb{N},
\end{equation*}
Therefore, 
\[
    \norm{\Pi_\mathcal{Y} \widehat{T}_{\epsilon_n} \Pi_\mathcal{X}}_{S_p} \leq \norm{\widehat{T}_{\epsilon_n}}_{S_p} \leq B,
\]
so that $\Pi_\mathcal{Y} \widehat{T}_{\epsilon_n} \Pi_\mathcal{X} \in \mathcal{T}_p$ is feasible for the ERM.

Consider an optimization over the set of finite rank operators $\mathcal{T}_N \bydef \set{\Pi_\mathcal{Y} T \Pi_\mathcal{X}: T \in \mathcal{T}_p} \subset \mathcal{T}_p$. From \eqref{eq:T_seq}, we have
\begin{equation*}
\inf_{T \in \mathcal{T}_N }J(T) \leq \lim_{n \rightarrow +\infty} J(\Pi_\mathcal{Y} \widehat{T}_{\epsilon_n} \Pi_\mathcal{X}) \leq \inf_{T \in \mathcal{T}_p }J(T).
\end{equation*}
Since $\mathcal{T}_{N}$ is a finite-dimensional set isometric to $N \times N$ matrices, and a closed ball in finite dimension is compact, this optimization is over a compact set and the minimum is achieved. In summary,
\begin{equation*}
\min_{T \in \mathcal{T}_N }J(T) = \inf_{T \in \mathcal{T}_p }J(T).
\end{equation*}

From the definition $\mathcal{T}_N$, any $T \in \mathcal{T}_N $ can be written as
\begin{equation*}
T = \sum_{i=1}^{N_y} \sum_{j=1}^{N_x} \alpha_{i, j} v_{i} u_{j}^*,
\end{equation*}
where $(u_1, \cdots , u_{N_x})$ and $(v_1, \cdots , v_{N_y})$ are two sets of complete orthonormal bases for linear subspaces $\mathcal{X}$ and $\mathcal{Y}$, and $\alpha_{i,j} \in \C$. Hence, we can identify $T$ with a matrix $\mT = (\alpha_{i, j}) \in \mathbb{C}^{N_y \times N_x}$. Similarly, $x_n \in \mathcal{X}$ and  $y_n \in \mathcal{Y}$ can be represented as vectors $\vx_n \in \mathbb{C}^{N_x}$ and $\vy_n \in \mathbb{C}^{N_y}$.  With this notation, we can implement ERM via a finite-dimensional optimization in \Cref{alg:ERM}.

\begin{algorithm}
\caption{Finite Dimensional Equivalent of ERM}
\label{alg:ERM}
\begin{algorithmic} 
\STATE {\bfseries Input:} $(x_n, y_n)_{n=1}^{N}$
\STATE {\bfseries Compute:} Basis functions $(u_1, \cdots, u_{N_x})$ and $(v_1, \cdots, v_{N_y})$, and vectors $(\vx_n,\vy_n)_{n=1}^{N}$.
\STATE {\bfseries Solve:} \begin{equation*}
\widehat{\mT}_N = \argmin_{\mT \in \C^{N_y \times N_x}: \| \mT \|_{S_p} \leq B} \frac{1}{N} \sum_{n=1}^{N} \mathcal{L} (\vy_n,\mT\vx_n).
\end{equation*}
\STATE {\bfseries Output:}  $\widehat{T}_N = \sum_{i=1}^{N_y} \sum_{j=1}^{N_x} \alpha_{i, j} v_{i} u_{j}^*, $ where $(\alpha_{i,j}) := \widehat{\mT}_N$.
\end{algorithmic}
\end{algorithm}

\section{Proofs of Formal Results} \label{sec:proofs}

\subsection[Proof of Learnability Theorem]{Proof of Theorem~\ref{thm:learnability_theorem}}

We define the function class $\mathcal{F}$ as
\begin{equation} \label{eq:F}
\mathcal{F} := \{f: \mathcal{H}_x \times \mathcal{H}_y \rightarrow \R^{+}: f(z) = \mathcal{L}(y,Tx), T \in \mathcal{T}_p \},
\end{equation}
where $z := (x, y)$, and adopt the following simplifying notations:
\begin{align*}
\hspace{2mm}&f(z) &:=&&& \mathcal{L}(y , Tx) \\
&P(f) &:=&&& \E_P [f] \\
&\widehat{f}_{N} &:=&&& \argmin_{f \in \mathcal{F}} P_N(f) 
:= \argmin_{f \in \mathcal{F}} \frac{1}{N} \sum_{n} f(z_n) \\
&\mathcal{L}^*(\mathcal{F}) &:=&&& \inf_{f \in \mathcal{F}} P(f) \\
&\| P - P^{'} \|_{\mathcal{F}} \hspace{-3mm} &:=&&& \hspace{-3mm} \sup_{f \in \mathcal{F}} | P(f) - P^{'} (f)|.
\end{align*}

As we mentioned before, the existence of the minimizer in the definition of $\wh{f}_N$ follows from the discussion in Section \ref{sec:learning_algorithm}.

The following known result bounds the empirical risk in terms of a quantity called the uniform deviation $\Delta_N$, which measures maximum discrepancy between empirical and generalized risks among the elements of the hypothesis class.

\begin{lemma} \label{lem:mismatched_minimization}
The empirical risk minimization (ERM) algorithm satisfies the following inequality:
\begin{equation*}
P(\widehat{f}_N) \leq \calL^{*}(\mathcal{F}) + 2\Delta_N(z^N)
\end{equation*}
where $\Delta_N (z^N) = \| P_N - P \|_{\mathcal{F}}$ is the uniform deviation for a function class $\mathcal{F}$ and given sequence of random samples, $z^N := (z_1, \cdots, z_N)$, \cite{shalev2014understanding}.
\end{lemma}
\begin{proof}
The proof is standard up to a modification due to the noncompactness of our $\calF$.
Let $f_\epsilon^* \in \calF$ be such that $P(f_\epsilon^*) - \calL^*(\calF) \leq \epsilon$. We have, 
\begin{align*}
    P(\widehat{f}_N) - \calL^{*}(\mathcal{F}) 
    &= P(\widehat{f}_N) - P_N(\widehat{f}_N) + P_N(\widehat{f}_N) - P_N(f_\epsilon^*) \\ 
    &+ P_N(f_\epsilon^*)- P(f_\epsilon^*) + P(f_\epsilon^*) - \calL^*(\calF) \\
    &\leq 2\Delta_N(z^N) + P(f_\epsilon^*) - \calL^*(\calF),
\end{align*}
since the first and the third bracketed term are bounded by $\Delta_N(z^N)$ by definition, and the third term is negative because $P_N$ is minimized by $\wh{f}_N$.
The conclusion follows from the fact that we can find $f_\epsilon^*$ such that $P(f_\epsilon^*) - \calL^*(\calF) \leq \epsilon$ for any $\epsilon > 0$ and that 
\[
    P(f_\epsilon^*) - \calL^*(\calF) \stackrel{\epsilon \to 0}{\longrightarrow} 0
\]
\end{proof}
We want to find a probabilistic upper bound for the uniform deviation. This can be achieved by showing $\Delta_N (z^N)$ has the bounded difference property and using McDiarmid's inequality.
\begin{lemma}
Let $P$ be a probability measure on a set $\mathcal{Z}$ and suppose that $z_1, \cdots, z_N$ and $\bar{z}_1, \cdots, \bar{z}_i, \cdots, \bar{z}_N$ are independent samples selected according to $P$. Define $z^N = z_1, \cdots, z_N$ and $z^N_{\bar{i}} = z_1, \cdots, \bar{z}_{i}, \cdots, z_N$. Then,
\begin{equation*}
|\Delta_N(z^N) - \Delta_N(z^N_{\bar{i}})| \leq \frac{1}{N} (C_y + BC_x)^2
\end{equation*}
for the function class $\mathcal{F}$ in \cref{eq:F}.
\end{lemma}
\begin{proof}
Define $P_{N_{\bar{i}}} = \frac{1}{N}\sum_{n=1}^{N}{f(z_n)} - \frac{1}{N}(f(z_i)-f(\bar{z}_i))$. 
\begin{align*}
&|\Delta_N(z^N) - \Delta_N(z^N_{\bar{i}})| \stackrel{\text{def.}}{=}  | \| P_N - P \|_{\mathcal{F}} - \| P_{N_{\bar{i}}} - P \|_{\mathcal{F}} | \\
&= \| P_{N_{\bar{i}}} + \frac{1}{N}f(z_i)-\frac{1}{N}f(\bar{z}_i)-  P \|_{\mathcal{F}} - \| P_{N_{\bar{i}}} - P \|_{\mathcal{F}} | \\
&\stackrel{\text{(a)}}{\leq}  \frac{1}{N}\|f(\bar{z}_i) -f(z_i) \|_{\mathcal{F}} \\
&\stackrel{\text{(b)}}{\leq}  \frac{1}{N} \sup_{T \in \mathcal{T}_p} \mathcal{L}(\bar{y}_i, T\bar{x}_i) + \frac{1}{N}\sup_{T \in \mathcal{T}_p} \mathcal{L}(y_i, Tx_i) \\
&\stackrel{\text{(c)}}{\leq} \frac{2}{N} (C_y + BC_x)^2,
\end{align*}
where $\text{(a)}$ and $\text{(b)}$ are due to triangle inequality. Finally, supremum in $\sup_{T \in \mathcal{T}_p} \mathcal{L}(y, Tx)$ is achieved by $$\widehat{T}= -\frac{B}{\|y \| \|x\|} y x^* \in \mathcal{T}_p$$ which gives us inequality $\text{(c)}$.
\end{proof}
Now, McDiarmid's inequality gives us a bound for tail probability of $\Delta_N(z^N)$ as
\begin{equation*}
\mathrm{P}(\Delta_N(z^N) \geq \E \Delta_N(z^N) + t) \leq \exp \{-\frac{Nt^2}{(C_y + BC_x)^4} \}.
\end{equation*}
Finally, we conclude the proof by choosing $$t = (C_y+ BC_x)^2 \sqrt{\frac{1}{N}\log{\frac{1}{\delta}}},$$ and using the following theorem:
\begin{theorem}
\cite{vapnik2015uniform, gine1984some} Fix a space $\mathrm{Z}$ and let $\mathcal{F}$ be a class of functions. Then for any probability measure $P \in \mathcal{P}(\mathrm{Z})$
\begin{equation*}
\E \Delta_N(z^N) \leq 2 R_N(\mathcal{F}).
\end{equation*}
\end{theorem}
\subsection[Proof of Vanishing Rademacher Complexity]{Proof of Theorem~\ref{thm:vanishing_rademacher_complexity}}
A simple analysis gives us an upper bound for $R_N(\mathcal{L} \circ \mathcal{T}_p)$.

\begin{align*}
& R_N(\mathcal{L} \circ  \mathcal{T}_p) \stackrel{\text{(a)}}{=} \E \sup_{T \in \mathcal{T}_p} |\frac{1}{N} \sum_{n=1}^{N}{\sigma_n \|y_n -Tx_n \|^2}| \\
&\stackrel{\text{(b)}} {\leq} \frac{1}{N}\E|\sum_{n=1}^{N}{\sigma_n \| y_n \|^2}| + \frac{1}{N} \E \sup_{T \in \mathcal{T}_p} |\sum_{n=1}^{N}{\sigma_n \| Tx_n \|^2}| \\
&+\frac{2}{N} \E \sup_{T \in \mathcal{T}_p} |\sum_{n=1}^{N}{\sigma_n \text{Re}\langle T x_n, y_n \rangle}| \\
&\stackrel{\text{(c)}}{\leq} \frac{1}{N} \sqrt{\E \sum_{n=1}^{N} \| y_n \|^4 }+ \frac{1}{N} \E \sup_{T \in \mathcal{T}_p}  |\mathrm{Tr}( T^* T\sum_{n=1}^{N}{\sigma_n x_n x_n^*})| \\
&+ \frac{2}{N} \E \sup_{T \in \mathcal{T}_p} \mathrm{Tr}(T \sum_{n=1}^{N}{\sigma_n x_n y_n^*})| \\
&\stackrel{\text{(d)}}{\leq} \frac{1}{\sqrt{N}} \sqrt{\E \| y \|^4 } + \frac{1}{N} \E \sup_{T \in \mathcal{T}_p} \| T^* T\|_{S_p} \| \sum_{n=1}^{N}{\sigma_n x_n x_n^*}\|_{S_q} \\
&+ \frac{2}{N} \E\| T \|_{S_p} \| \sum_{n=1}^{N}{\sigma_n y_n x_n^*} \|_{S_q} \\
&\stackrel{\text{(e)}}{\leq} \frac{1}{\sqrt{N}} \sqrt{\E \| y \|^4 }+ \frac{B^2}{N}\E \| \sum_{n=1}^{N}{\sigma_n x_n x_n^*}\|_{S_q} \\
&+ \frac{2B}{N} \E  \| \sum_{n=1}^{N}{\sigma_n y_n x_n^*} \|_{S_q}
\end{align*}
The equality (a) simply follows from the definition of expected Rademacher complexity of $\mathcal{L} \circ  \mathcal{T}_p$, see \Cref{def:rademacher_complexity}, and (b) from triangle inequality and subadditivity property of supremum. In (c), we used Jensen's inequality along with $\E \sigma_n \sigma_m = \delta (m-n)$ which leads to $\E|\sum_{n=1}^{N}{\sigma_n \| y_n \|^2}| \leq \sqrt{\E \sum_{n=1}^{N} \| y_n \|^4 }$. 

The following results explain the remaining inequalities.

\begin{lemma} \label{lem:adoint_operator}
Let $x \in \mathcal{H}_x$ and $y \in \mathcal{H}_y$ with bounded norms, and $T \in S_{\infty}(\mathcal{H}_x, \mathcal{H}_y)$ be a bounded operator. Then,
\begin{itemize}
\item $\| Tx \|^2 = \mathrm{Tr}(T^*Tx x^*)$,
\item $\langle Tx , y \rangle = \mathrm{Tr} (Txy^*)$.
\end{itemize}
\end{lemma}
\begin{proof}
Define $e_x = \frac{1}{\|x\|}x$ and $e_y = \frac{1}{\|y\|}y$. Let $e_x \cup \{ e_{x,n} \}$, $e_y \cup \{ e_{y,n} \}$ be a set of orthonormal basis for $\mathcal{H}_x$ and $\mathcal{H}_y$. Then,
\begin{align*}
\mathrm{Tr}(T ^*T xx^* ) &= \sum_{n} \langle T^* Tx x^*e_{x,n}, e_{x,n}\rangle + \langle T^* Tx x^*e_x, e_{x}\rangle \\
&= \langle T^*Tx \|x \|, e_{x}\rangle \\
&= \langle T^*Tx , x\rangle\\
&= \|Tx\|^2,	
\end{align*}
and 
\begin{align*}
\mathrm{Tr} (Txy^*) &=  \sum_{n} \langle Tx y^*e_{y,n}, e_{y,n}\rangle + \langle Txy^* e_y, e_{y}\rangle\\
&= \langle Tx\|y\|, e_{y} \rangle \\
&= \langle Tx , y\rangle
\end{align*}
\end{proof}

The second part of $(\text{c})$ is now a direct consequence of \Cref{lem:adoint_operator}. Note that $S_{p}(\mathcal{H}_x, \mathcal{H}_y) \subset S_{\infty}(\mathcal{H}_x, \mathcal{H}_y)$ and that $\mathrm{Tr}(\cdot)$ is a linear operator. The following theorem explains inequality $(\text{d})$.
\begin{theorem}
\cite{reed1975ii} Let $1 \leq p \leq \infty$ and $p^{-1}+ q^{-1} = 1$. If $T_1 \in S_p(\mathcal{H}_z, \mathcal{H}_y)$ and $T_2 \in S_q(\mathcal{H}_x, \mathcal{H}_z)$ then 
\[  
    T_1^*T_2 \in S_{1}(\mathcal{H}_x, \mathcal{H}_y)
\]
and 
\[
    |\mathrm{Tr} ( T_1^* T_2)| \leq \norm{T_1^*T_2}_{S_1} \leq \|T_1\|_{S_p} \|T_2\|_{S_q}.
\]
\end{theorem}
\noindent Finally, $\| T^* T \|_{S_p} = \| T \|^2_{S_{2p}} \leq\| T \|^2_{S_p}$ proves the inequality $(\text{e})$.

We want to bound the expected Schatten $q$-norms of random operators $T_{xx}  := \sum_{n=1}^{N}{\sigma_n x_n x_n^*}$ and $T_{yx}  := \sum_{n=1}^{N}{\sigma_n y_n x_n^*}$. From \Cref{lem:random-op},
\begin{equation*}
\E \| T_{xx}\|_{S_q} \leq  N^{\max\{ \frac{1}{2}, \frac{1}{q}\}}\sqrt{ \E  \| x\|^4  },
\end{equation*}
\begin{equation*}
\E \| T_{yx} \|_q \leq  N^{\max\{ \frac{1}{2}, \frac{1}{q}\}}\sqrt{ \E  \| x\|^2 \| y\|^2  }.
\end{equation*}
Since $\| x_n \|$ and $\| y_n \|$ are bounded random variables, we conclude
\begin{align*}
R_N(\mathcal{L} \circ  \mathcal{T}_p) \leq \frac{C_y}{\sqrt{N}}+ N^{-\min\{ \frac{1}{2}, \frac{1}{p}\}}(B^2 C_x  + 2BC_x C_y).
\end{align*}

\subsection{Proof of Lemma \ref{lem:random-op}}

To prove Lemma \ref{lem:random-op}, we start by stating two results about compact operators. Let $\text{Spec}(A)$ denote the spectrum of an operator $A$. The we have:

\begin{lemma} \label{lem:convexity} \cite{petz1994survey}
Let $f : [\alpha, \beta] \rightarrow \mathbb{R}$ be convex. Then the functional
$F(A) = \mathrm{Tr} ( f(A) )$ is convex on the set $\{A \in \mathcal{T} : \text{Spec}(A) \in [\alpha, \beta]\}$ and $\mathcal{T}$ is the set of finite rank, self-adjoint operators.
\end{lemma}

The following lemma is a standard application of H\"older's inequality.

\begin{lemma} \label{lem:lp_to_lq}
Let $T$ be a non-negative, compact linear operator of rank at most $N$. Then, 
\begin{equation*}
\| T \|_{S_p} \leq N^{p^{-1}-q^{-1}}\| T \|_{S_q}
\end{equation*}
where $q \geq p \geq 1$.
\begin{proof}
Let $s_1(T) \geq \cdots \geq s_{N}(T) \geq 0$ be the sequence of singular values of $T$ (with multiplicities), and $q > p \geq  1$. Then,
\begin{align*}
\| T \|_{S_p} &= \big( \sum_{n=1}^{N} s_n(T)^{p} \big)^{\frac{1}{p}} \\
&\stackrel{\text{H\"older ineq.}}{\leq}\bigg( \big(\sum_{n=1}^{N} s_n(T)^{q}\big)^{\frac{p}{q}} \big(\sum_{n=1}^{N} 1^{\frac{q}{q-p}}\big)^{1- \frac{p}{q}} \bigg)^{\frac{1}{p}}\\
&= N^{p^{-1}-q^{-1}} \| T \|_{S_q} .
\end{align*}
\end{proof}
\end{lemma}



We now proceed to prove Lemma \ref{lem:random-op}. If $1 \leq q < 2$, we have:
\begin{align*}
\E \| T_{xx} \|_q &\stackrel{\text{(a)}}{=} \E_{x} \E_{\sigma} (\mathrm{Tr}\big( (T_{xx}T_{xx}^{*})^{q/2} \big) )^{1/q} \\
&\stackrel{\text{(b)}}{\leq} \E_{x} (\mathrm{Tr}\big( ( \E_{\sigma}  T_{xx}T_{xx}^{*})^{q/2} \big) )^{1/q}\\
&\stackrel{\text{(c)}}{=} \E_{x} (\mathrm{Tr}\big( (\sum_{n=1}^{N}{\| x_n\|^2 x_n x_n^* } )^{q/2} \big) )^{1/q}\\
&\stackrel{\text{(d)}}{=} \E_{x} \sqrt{ \| \sum_{n=1}^{N} \| x_n\|^2 x_n x_n^*\|_{S_{q/2}} } \\
&\stackrel{\text{(e)}}{\leq} \E_{x} \sqrt{ N^{\frac{2}{q}-1} \mathrm{Tr}\big( \sum_{n=1}^{N} \| x_n\|^2 x_n x_n^*\big) } \\
&\stackrel{\text{(f)}}{\leq}  N^{1/q}\sqrt{ \E  \| x\|^4  }.
\end{align*}
We start from $(a)$, the definition of Schatten norm of operators,  \cref{eq:Sp_norm}; (b) follows from \Cref{lem:convexity} as $f(A) = \mathrm{Tr} ( A^{\frac{q}{2}} )$ is a concave functional for $1 \leq  q < 2$. The equality (c) follows from $\E_{\sigma} \sum_{n,m =1}^{N} \sigma_n \sigma_m x_n x_n^* x_m x_m^*= \sum_{n=1}^{N} \|x_n\|^2 x_n x_n^*$. The equality $(d)$ combines the definition of $\norm{\cdot}_{S_p}$ norm and the fact that $\sum_{n=1}^{N} \| x_n\|^2 x_n x_n^*$ is a non-negative, self-adjoint operator. The inequality (e) follows from the bound provided in \Cref{lem:lp_to_lq}, where we upper bound the operator $S_{q/2}$ norm by that of $S_1$, and finally (f) is due to the concavity of the squre root and the fact that $\{x_n\}_{n=1}^{N}$ are identically distributed and $\mathrm{Tr} (\|x\|^2 x x^*) = \| x\|^4$.

If $q \geq 2$, we have:
\begin{align*}
\E \| T_{xx} \|_q &\stackrel{\text{(a)}}{\leq} \E_{x}\E_{\sigma} (\mathrm{Tr}\big( T_{xx}T_{xx}^{*} \big) )^{\frac{1}{2}} \\
&\stackrel{\text{(b)}}{\leq} \E_{x} (\mathrm{Tr}\big( \E_{\sigma} T_{xx}T_{xx}^{*} \big) )^{\frac{1}{2}}\\
&= \E_{x} (\mathrm{Tr}\big( \sum_{n=1}^{N}{\| x_n\|^2 x_n x_n^* }  \big) )^{\frac{1}{2}}\\
&\stackrel{\text{(c)}}{\leq}  (\E_{x}\sum_{n=1}^{N}{\| x_n\|^4 }  )^{\frac{1}{2}}\\
&\stackrel{\text{(d)}}{\leq}  N^{1/2}\sqrt{ \E  \| x\|^4  }
\end{align*}
The inequality $(\text{a})$ is due to the fact that for any operator $T \in S_{p}(\mathcal{H}_1, \mathcal{H}_2) \cap S_{q}(\mathcal{H}_1, \mathcal{H}_2)$, we have $\| T \|_{p} \leq \| T \|_{q}$ if $p \geq q$;  (b) follows from Jensen's inequality and concavity of $x \mapsto \sqrt{x}$; (c) from the linearity of $\mathrm{Tr}(\cdot)$ and $\E[\cdot]$ and the fact that $\mathrm{Tr} (\|x\|^2 x x^*)= \| x\|^4$; and (d) 
from that fact that the $x_i$ are identically distributed.

\section{Conclusion}

We studied the complexity of learning Schatten--von Neumann operators. Our results are the first we know of that give guarantees for learning infinite-dimensional non-compact classes of operators. We show that for $p < \infty$ these operators are indeed learnable. Our motivation comes primarily from applications of machine learning in data-driven approaches to inverse problems, which are fundamentally problems of regressing infinite dimensional signals such as images, from infinite dimensional data such as sinograms (in computed tomography). In that context, our results imply scaling laws between the number of training samples and the  approximation error guaranteed for the reconstructed images. We see this contribution as a first in a series that establishes guarantees for applications of machine learning to inverse problems.

\bibliographystyle{IEEEbib}
\bibliography{opsamp}

\end{document}